\documentclass{article}

\usepackage{PRIMEarxiv}

\usepackage[utf8]{inputenc} 
\usepackage[T1]{fontenc}    
\usepackage{hyperref}       
\usepackage{url}            
\usepackage{booktabs}       
\usepackage{amsfonts}       
\usepackage{nicefrac}       
\usepackage{microtype}      
\usepackage{lipsum}
\usepackage{fancyhdr}       
\usepackage{graphicx}       
\graphicspath{{media/}}     
\usepackage{algorithm}
\usepackage{color}
\usepackage{algpseudocode}
\usepackage{amsmath}
\usepackage{amsthm}
\newtheorem{theorem}{Theorem}
\newtheorem{lemma}[theorem]{Lemma}

\title{MimicDiffusion: Purifying Adversarial Perturbation via Mimicking Clean Diffusion Model
}
\author{
  Kaiyu Song, Hanjiang Lai \\
  Sun Yat-Sen University \\
  \texttt{\{songky7, laihanj3\}@mail.sysu.edu.cn}
}

\begin{document}
\maketitle

\begin{abstract}
Deep neural networks (DNNs) are vulnerable to adversarial perturbation, where an imperceptible perturbation is added to the image that can fool the DNNs. Diffusion-based adversarial purification focuses on using the diffusion model to generate a clean image against such adversarial attacks. Unfortunately, the generative process of the diffusion model is also inevitably affected by adversarial perturbation since the diffusion model is also a deep network where its input has adversarial perturbation. In this work, we propose MimicDiffusion, a new diffusion-based adversarial purification technique, that directly approximates the generative process of the diffusion model with the clean image as input. Concretely, we analyze the differences between the guided terms using the clean image and the adversarial sample. After that, we first implement MimicDiffusion based on Manhattan distance. Then, we propose two guidance to purify the adversarial perturbation and approximate the clean diffusion model. 
Extensive experiments on three image datasets including CIFAR-10, CIFAR-100, and ImageNet with three classifier backbones including WideResNet-70-16, WideResNet-28-10, and ResNet50 demonstrate that MimicDiffusion significantly performs better than the state-of-the-art baselines. On CIFAR-10, CIFAR-100, and ImageNet, it achieves 92.67\%, 61.35\%, and 61.53\% average robust accuracy, which are 18.49\%, 13.23\%, and 17.64\% higher, respectively. The code is available in the supplementary material.
\end{abstract}    
\section{Introduction}
Deep neural networks (DNNs) have achieved great success in various fields of computer vision, e.g., image detection~\cite{diffuse}, image classification~\cite{advertrain:diffusion}. However, DNNs are vulnerable to the \textit{adversarial samples}~\cite{adattack:fgsd}, where the adversarial sample consists of the clean sample and an imperceptible adversarial perturbation.  

To defend against adversarial attack, \textit{adversarial training}~\cite{adtrain:AAAED,adtrain:IRGD} has been proposed by leveraging the generated adversarial samples to train the classifier. For example, Bai \textit{et al.}~\cite{baseline:bai2021} used the adversarial samples as the training data to train the classifier directly. However, adversarial training may be ineffective when suffering from unknown attack methods~\cite{advtrain_disadvantage}.

In contrast, another popular method for the adversarial attack is \textit{adversarial purification}~\cite{adpuri:baseline1,adpuri:baseline2,adpuri:baseline3}. Given an adversarial sample as the input, adversarial purification methods aim to purify the adversarial perturbation from the adversarial sample and finally obtain the generated clean sample. Then the generated clean sample is fed into the classifier.

As one of the popular generative models, the diffusion model~\cite{DDPM} becomes a potential tool for adversarial purification due to generating high-quality images. Previous methods~\cite{adpuri:baseline3,adpuri:baseline2} depended on finding an optimal time step in the forward process to cover the adversarial perturbation. Then, the reverse process tries to purify both the Gaussian noise and adversarial perturbation and keep the label semantic at the same time. Yong \textit{et al.}~\cite{adpuri:baseline1} proposed the score-based method by using the property that the clear sample tends to be the lower value of the score function.  Nie \textit{et al.} ~\cite{adpuri:baseline3} proposed the DiffPure based on a small-step denoiser. Further, there are some improved methods based on the iteration methods and the guided methods~\cite{GDPM}. For example, Wang \textit{et al.}~\cite{GDPM} alleviated the requirement for keeping the label semantic by incorporating a guidance~\cite{DPS}.

Despite the success of diffusion-based adversarial purification methods, we argue that adversarial perturbations added to the clean samples will still affect the generative process of the diffusion model, which deviates from the trajectory of the clean diffusion model. Thus, it will generate extra noise in the synthetic images and cause performance degradation (More details can be found in Section~\ref{Preliminary}). The key is to remove the effects of the adversarial perturbation when performing the generative process.

As shown in Fig.~\ref{fig:introduction}, an intuitive approach is that if the input is not an adversarial sample but a clean image, the adversarial perturbation problem would disappear. Therefore, an interesting idea arises: \textit{without knowing the clean inputs, can we mimic the trajectory of the diffusion model with clean inputs to reduce the effect of adversarial perturbations?} 

\begin{figure}
    \centering
    \includegraphics[width=0.5\textwidth]{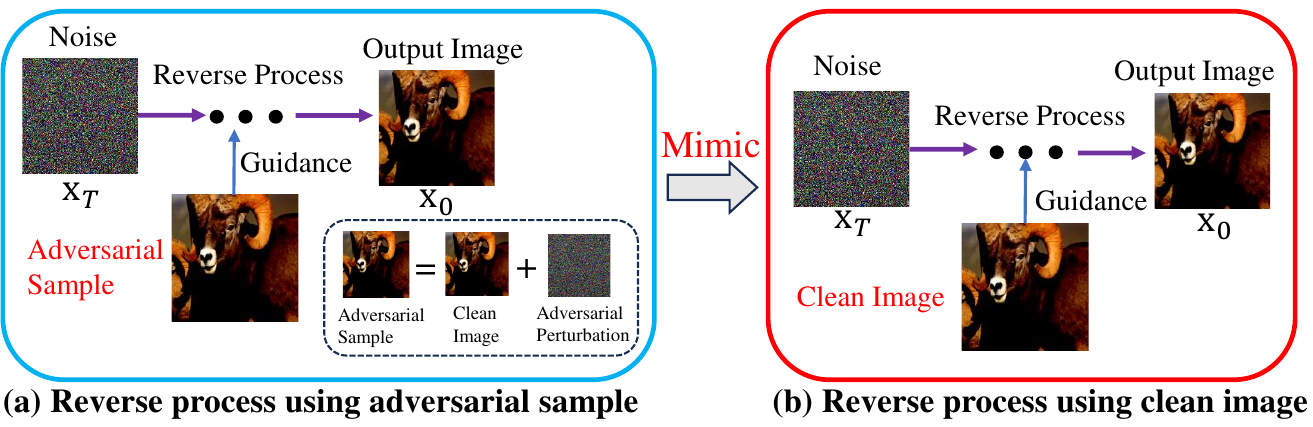}
    \caption{An illustration of MimicDiffusion. By implementing the guidance method and using the adversarial sample (clean image + adversarial perturbation), we aim to alleviate the influence of the adversarial perturbation for the reverse process of the diffusion model such that it is similar to the reverse process with the input of the clean image.} 
    \label{fig:introduction}
\end{figure}

In this work, we propose a novel MimicDiffusion to reduce the negative influence of adversarial perturbations. We use guided diffusion as a backbone, where the Gaussian noise is used as input and the adversarial sample that is composed of clean sample and adversarial perturbation is used as guidance. The main problem is that the guided term also includes adversarial perturbation. Fortunately, adversarial perturbations are very small values. Under this assumption, we first use Manhattan distance ($\ell_1$ distance) instead of Euclidean distance ($\ell_2$ distance). Thus we reduce the range of the derivative to +1 or -1. We further show that using Manhattan distance can be divided into two cases, short range, and long range, to compare the difference in derivatives of the adversarial sample and the clean sample. Concretely, 1) when the Manhattan distance between the generated image and the clean image is larger than the maximum value in adversarial perturbations, called the \textit{long-range distance}, we show that the gradients of the two guidance are the same. That is what we want. 2) When the Manhattan distance between the generated image and the clean image is smaller than the maximum value in adversarial perturbations, called the \textit{short-range distance}, the gradients of the two guidance may or may not be equal. 

According to the above observations, we propose two guidance: one for the long-range distance and another for the short-range distance. In the long-range guidance, we can simply use the adversarial sample as the guidance since the gradients are the same. In the short-range guidance, we propose a non-linear transform operation inspired by super-resolution measurement~\cite{DPS}. This method involves projecting the generated image and the adversarial image onto a higher dimensional space, effectively increasing the Manhattan distance between them beyond the maximum value of adversarial perturbations. Hence, it may be going to be the same as the case of the long-range distance. 

Ultimately, we compare our method to the latest adversarial training and adversarial purification methods on various strong adaptive attack benchmarks. Extensive experiments on CIFAR-10, CIFAR-100, and ImageNet across various classifiers such as WideResNet-28-10 and WideResNet-70-16 show that MimicDiffusion achieves state-of-the-art performance. Compared with the latest adversarial purification methods~\cite{adpuri:baseline1}, e.g., AutoAttck ($\ell_{\infty}, \epsilon=8/255$)~\cite{autoattack}, we show the absolute improvement of $+18.49\%$, $+13.23\%$ and $17.64\%$ in average robust accuracy on CIFAR-10, CIFAR-100 and ImageNet with WideResNet-28-10 respectively. 

 To sum up, the main contributions of this paper are:
 \begin{itemize}
    \item We propose a new perspective for diffusion-based adversarial purification methods, which mimic the generative process of the diffusion model with the clean image as input to reduce the negative influence of adversarial perturbation.
     \item We propose a novel MimicDiffusion, where we use Manhattan distance and propose long-range guidance and short-range guidance to bridge the gap between the clean sample and the adversarial sample.
     \item The experimental results show that our model achieves state-of-the-art performance on various adaptive attack benchmarks.
 \end{itemize}
\section{Related Work}
\textbf{Adversarial training} uses the adversarial samples to train the classifier~\cite{relatedwork:firstadversarialtrain}. Thus, the classifier can also correctly recognize the adversarial samples. The proposed method~\cite{relatedwork:ODE-lyapunov-stable} uses ordinary differential equations to re-sample the feature point from the Lyapunov-stable equilibrium points. The data augmentation method \cite{relatedworkDataAug} was proposed to generate lots of adversarial samples and improve the robust prior knowledge.

\textbf{Adversarial purification} uses the generative models to generate clean images and could be regarded as the classifier-agnostic method. These methods try to directly remove the adversarial perturbation in the adversarial sample. Thus, the classifier can be fixed. Pouya \textit{et al.}~\cite{relatework:puri-gan} used the generative adversarial network (GANs) to purify the perturbation. Meanwhile, the score-based match model \cite{relatework:sde,relatework:scorebased} was proposed to eliminate the influence of the perturbation and recover a clean image based on the score-based match network. We empirically compare our method with the previous works and the experimental results show that our method can obtain significant improvements.

\textbf{Diffusion model based adversarial purification}. Recently, Nie \textit{et al.} \cite{adpuri:baseline3} proposed the diffusion-based adversarial purification method, which proved that generating images from the diffusion model will tend to be the clean image. Therefore, one entire step of the diffusion process could purify the adversarial perturbation. However, limited by the different designs in different types of diffusion models \cite{Karras2022edm}, it is difficult to reach the optimal performance under adversarial purification. To alleviate this, GDPM~\cite{GDPM} proposed an iteration-based method and incorporated the guided method~\cite{DPS} to keep the label semantic under the iteration process. Proven by Lee \textit{et al.}~\cite{RobustEvaluation}, GDPM and Nie \textit{et al.} rely on finding an optimal hyperparameter setting, e.g. the optimal time step. Eventually, adversarial perturbation still has a negative influence on the adversarial purification methods.

\section{Preliminary} 
\label{Preliminary}

 We first give some definitions. In the adversarial purification, we have a diffusion model, e.g., the score function $s_{\theta}(x_t) = \nabla_{x} \log p(x;t)$, that is trained on the original dataset. Now given the adversarial sample denoted as $x^{adv}$, where $x^{adv} = x^{ori} + \phi$ and $x^{ori}$ is the clean image (unknown) and $\phi$ is the adversarial perturbation (unknown) generated by adversarial attack method, adversarial purification aims to recover the clean image $x^{ori}$ from the input adversarial sample $x^{adv}$.

\textbf{Diffusion model based adversarial purification}. The idea is to remove both the adversarial purification and Gaussian noise in the reverse process of the diffusion model. 
First, it finds a time step $t^\ast$ such that:
\begin{equation}
\begin{split}
    x_{t^\ast} & = \sqrt{\sigma(t^\ast)}x^{adv} + \sqrt{(1-\sigma(t^\ast))} \epsilon \\
    &= \sqrt{\sigma(t^\ast)}(x^{ori} + \phi) + \sqrt{(1-\sigma(t^\ast))} \epsilon,
\end{split}
\label{eq:purification_adding_noise}
\end{equation}
where it is the forward process of diffusion model~\cite{DDPM}, $x_{t^\ast}$ is the state in the $t^\ast$ time, $\sigma(\ast)$ is the noise schedule related to the time step $t$, and $\epsilon \sim \mathcal{N}(0,1)$ is the Gaussian noise. Then, the backward process of diffusion model~\cite{DDPM} is performed on $x_{t^\ast}$ to generate the clean image $\hat{x}^{ori}$.

In the above approaches, the key to success is finding an optimal $t^{\ast}$~\cite{adpuri:baseline3}, and thus the performance is sensitive to the value of $t^{\ast}$.

\textbf{Guided-diffusion based adversarial purification}. One of the state-of-the-art methods is the guided diffusion model~\cite{GDPM}. 
In the guided diffusion methods, the adversarial sample $x^{adv}$ is used as guidance, and it starts from the pure Gaussian noise $x_T$ in the backward process. In the $t$ time step, the guided generating process is formulated as:
\begin{equation}
    \nabla_{x} \log p(x_t|x^{adv};t) = \underbrace{\nabla_{x} \log p(x_t;t)}_{\text{Score Function}} +  \underbrace{\nabla_{x} \log p(x^{adv}|x_t;t)}_{\text{Guidance Term}}.
    \label{eq:guidedmethoddefination}
\end{equation}
The score function is already known and the guidance term can be approximated as~\cite{DPS}:
\begin{equation}
\begin{split}
    \nabla_{x} \log p(x^{adv}|x_t;t) &= -R_{t}\nabla_{x_{t}}d(\hat{x}_{t},x^{adv}), \\
    \hat{x}_{t} &= \frac{x_{t}-\sqrt{1-\sigma(t)} s_{\theta}(x_{t})}{\sqrt{\sigma(t)}},
\end{split}
    \label{eq:DPS guided function}
\end{equation}
where $s_{\theta}(x_{t})$ is the known score function~\cite{relatework:sde} with the parameter $\theta$ for $x_{t}$ in the $t$ time, $\hat{x}_{t}$ is the estimation for $x_{0}$ in the $t$ time, $R_{t}$ is the guided factor related to the $t$ time, and $d(\ast,\ast)$ is the $\ell_2$ norm distance metric.

\textbf{Motivation}. However, we argue that the trajectory of the guided method will still be influenced by the adversarial perturbation. 
To show this, considering the Eq.~\ref{eq:DPS guided function} with the $\ell_2$ distance norm, we have
\begin{equation}
\begin{split}
    -R_{t}\nabla_{x_{t}}d(\hat{x}_{t},x^{adv}) &= -R_t\nabla_{x_{t}}||\hat{x}_{t} - x^{adv}||_2^2 \\
            &= -R_t\nabla_{x_{t}}||\hat{x}_{t} - x^{ori} - \phi||_2^2 \\
            &= -R_t\frac{\partial||\hat{x}_{t} - x^{ori} - \phi||_2^2}{\partial \hat{x}_{t}}\frac{\partial\hat{x}_{t}}{\partial x_{t}}.
\end{split}
\label{eq:l2norm}
\end{equation}
Then, the Jacobi matrix for the partial part is:
\begin{equation}
\begin{split}
    \frac{\partial||\hat{x}_{t} - x^{ori} - \phi||_2^2}{\partial \hat{x}_{t}} &= \nabla_{x_{t}}J((\hat{x}_{t} - x^{ori} - \phi)^2)\\
    &= 2 J(\hat{x}_{t} - x^{ori} - \textcolor{red}{\phi}),
\end{split}
\label{eq:reason}
\end{equation}
where $J(\ast)$ is the operation to calculate the Jacobi matrix. From Eq.~\ref{eq:reason}, we can see that the gradient of the guidance term also includes the adversarial perturbation $\phi$. Hence, the adversarial perturbation still affects the generative process of the guided diffusion model, which would cause the generated trajectory to deviate from the correct direction.

Suppose that we can replace the adversarial input to the clean sample $x^{ori}$, according to Eq.~\ref{eq:l2norm} - Eq.~\ref{eq:reason}, we have
\begin{equation}
R_{t}\nabla_{x_{t}}d(\hat{x}_{t},x^{ori}) \propto 2 R_{t} J((\hat{x}_{t} - x^{ori})\frac{\partial\hat{x}_{t}}{\partial x_{t}}).
    \label{eq:ori}
\end{equation}
Hence, if we can mimic the diffusion model with clean images as inputs, we can remove the negative influence of the adversarial perturbation.

\begin{figure*}[h!]
    \centering
    \includegraphics[width=\textwidth]{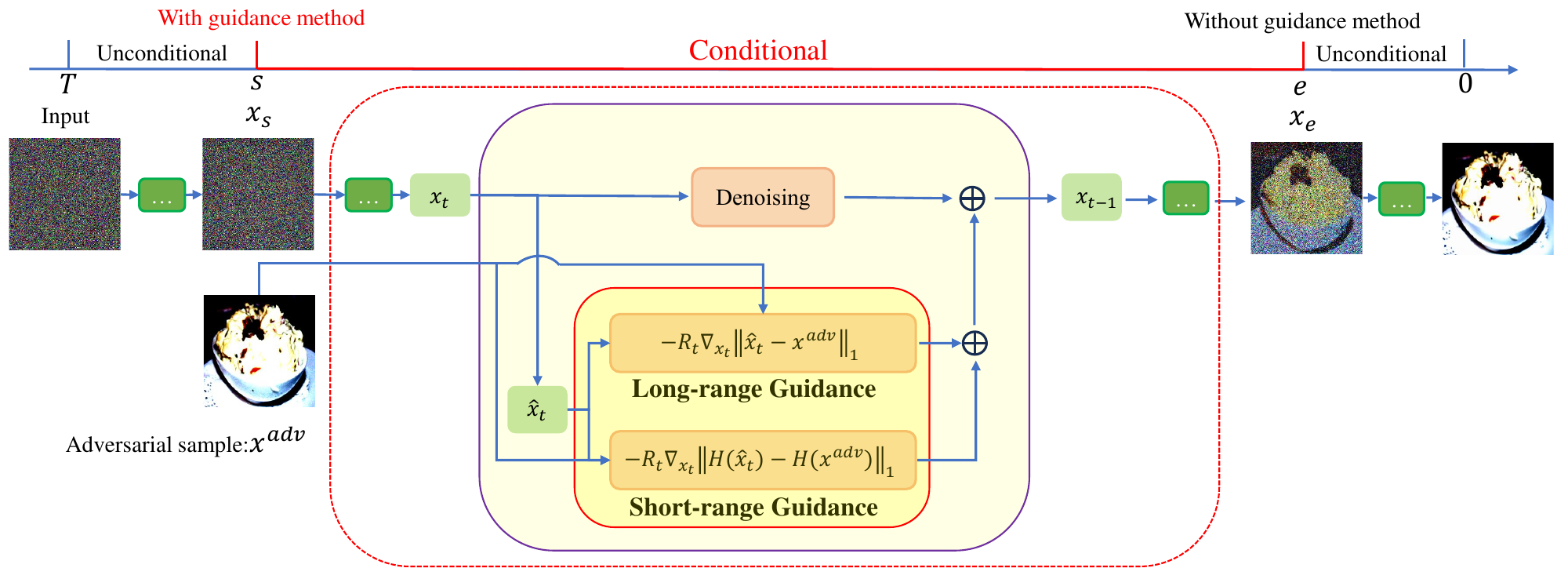}
    \caption{An overview of the proposed MimicDiffusion, where $x_{T}$ is the pure Gaussian noise, $[s,e]$ is the interval to implement the guidance method noted as conditional by using $x^{adv}$ as the measurement, the other time step without the guidance method noted as unconditional, denoise is one step reverse process, long-range guidance is used to eliminate the adversarial perturbation in long-range condition, and short-range guidance is used to alleviate the adversarial perturbation in short-range condition.} 
    \label{fig:MGDM}
\end{figure*}

\section{Method} 
In this paper, we aim to remove the effect of the adversarial perturbation on the guided diffusion model. According to our motivation, we want to approximate the gradients of the guided terms with the adversarial sample and clean sample as inputs:

\begin{equation} 
\nabla_{x_{t}}d(\hat{x}_{t},x^{adv}) \approx \nabla_{x_{t}}d(\hat{x}_{t},x^{ori}).
\end{equation}

\textbf{Manhattan distance}. For Euclidean distance, it is difficult to approximate the derivatives for different inputs. Fortunately, a common distance metric, i.e., Manhattan distance, might have the same gradients for different inputs. We denote $||x||_{\min} = \min(|x_1|,|x_2|,\cdots,|x_n|)$, where $n$ is the number of values in $x$.

\begin{lemma}
\label{lemma1}
Let $||\phi||_{\infty} < \xi$ and $x^{adv} = x^{ori} + \phi$, then we have the following relations for any $x_t$: 1) when $||x_{t}-x^{ori}||_{\min} > \xi$, we have $\nabla_{x_{t}}||x_{t}-x^{adv}||_1 = \nabla_{x_{t}}||x_{t}-x^{ori}||_1$; 2) when $||x_{t}-x^{ori}||_{\min} \leq \xi$, we have $\nabla_{x_{t}}||x_{t}-x^{adv}||_1 \overset{\text{Unknown}}{\longleftrightarrow} \nabla_{x_{t}}||x_{t}-x^{ori}||_1$.
\end{lemma}

\begin{proof}
    The proof is simple based on the derivative of $\ell_1$ can only be +1 or -1. We have:
\begin{equation}
\begin{split}
    \nabla_{x_{t}}||x_{t}-x^{adv}||_1 &= \nabla_{x_{t}}||x_{t}-x^{ori}-\phi||_1 \\
    &= \text{Sign}(x_{t}-x^{ori}-\phi),
\end{split}
\end{equation}
where $\text{Sign}(x) = 1$ if $x>0$ otherwise $\text{Sign}(x) = -1$. 
When $||x_{t}-x^{ori}||_{\min} > \xi$, we have $(x_{t}-x^{ori}) > \xi$ or $(-x_{t}+x^{ori}) > \xi$. Both the two cases, it is easy to verify that 
\begin{equation}
\begin{split}
   \text{Sign}(x_{t}-x^{ori}-\phi) &= \text{Sign}(x_{t}-x^{ori}).
\end{split}
\label{eq:absolutin}
\end{equation}
Hence, we have $\nabla_{x_{t}}||x_{t}-x^{adv}||_1 = \text{Sign}(x_{t}-x^{ori}-\phi) = \text{Sign}(x_{t}-x^{ori}) = \nabla_{x_{t}}||x_{t}-x^{ori}||_1$.

When $||x_{t}-x^{ori}||_{\min} < \xi$, the distance between $x_{t}$ and $x^{ori}$ is too close, and we cannot clarify the relation between $\nabla_{x_{t}}||x_{t}-x^{adv}||$ and $\nabla_{x_{t}}||x_{t}-x^{ori}||$.
\end{proof} 

Lemma~\ref{lemma1} shows two interesting observations. When $t$ is large and $x_{t}$ will tend to Gaussian noise, thus it satisfies $ ||x_{t}-x^{ori}||_{\min} > \xi$, and the gradients for the absolute value of $|x_{t} - x^{adv}|$ and $|x_{t} - x^{ori}|$ are equal. On the contrary, when $t$ is small and the distance between $x_{t}$ and $x^{adv}$ is too close, i.e., $||x_{t}-x^{ori}||_{\min} \leq \xi$, we cannot clarify the relation between $\nabla_{x_{t}}||x_{t}-x^{adv}||_1$ and $\nabla_{x_{t}}||x_{t}-x^{ori}||_1$.

Based on the two observations, we further propose MimicDiffusion shown in Fig.~\ref{fig:MGDM} to achieve mimicking. Concretely, we directly start from the Gaussian noise, which can avoid adding additional perturbation from the adversarial sample. The adversarial sample is only used as the guidance. The Manhattan distance is proposed in the guidance term to reduce the negative influence of adversarial perturbation. When the generated $x_t$ is far away from the $x^{ori}$, we call it long-range distance, we can simply use the adversarial sample as the guidance. Then, we propose to use a super-resolution operation to reduce the extra perturbation when $x_t$ is close to $x^{ori}$ called short-range distance. In the end, to better mimic the trajectory of the guidance with the $x^{ori}$ input, we further propose a novel sampling strategy to implement guidance in a particular time interval to reduce the extra noise and time cost at the same time.

\textbf{Long-range guidance}. In the long-range distance situation, e.g., $||\hat{x}_{t}-x^{ori}||_{\min} > \xi$, based on Lemma~\ref{lemma1}, we apply the Manhattan distance:
\begin{equation}
\begin{split}
    \frac{\partial||\hat{x}_{t} - x^{ori} - \phi||_1}{\partial \hat{x}_{t}} &= \nabla_{x_{t}}J(|\hat{x}_{t} - x^{adv}|)\\
    & = \nabla_{x_{t}}J(|\hat{x}_{t} - x^{ori}|) \\
     & = \frac{\partial||\hat{x}_{t} - x^{ori} ||_1}{\partial \hat{x}_{t}},
\end{split}
\label{eq:change}
\end{equation}
In this way, we have $\nabla_{x_{t}}d(\hat{x}_{t},x^{adv}) = \nabla_{x_{t}}d(\hat{x}_{t},x^{ori})$ that can eliminate $\phi$, and thus avoid adding extra adversarial perturbation from the guidance term. 

Hence, in the long-range distance, we can simply set the guidance $y^{l} = x^{adv}$, where the trajectory will be similar to the diffusion model with clear input. The long-range guidance can be formulated as:
\begin{equation}
    \nabla_{x_t} \log p(y^{l}|x_t) = -R_t\nabla_{x_{t}} ||\hat{x}_{t} - x^{adv}||_1,
    \label{eq:first term}
\end{equation}
where $\hat{x}_{t} = \frac{x_{t}-\sqrt{1-\sigma(t)} s_{\theta}(x_{t})}{\sqrt{\sigma(t)}}$ is the estimated image and $y^{l} = x^{adv}$.

\textbf{Short-range guidance}.
Based on the definition of short-range distance, in small time steps, e.g., the later phase of the reverse process, the unknown relation will eventually lead to the deviation of trajectory. To alleviate it, we choose the super-resolution operation in the short-range distance. The super-resolution operation is a non-linear mapping operation. Under the transform operation of super-resolution, the non-linear transform could increase the Manhattan distance to change the short-range distance to the long-range distance. In this condition, based on the Eq.\ref{eq:change}, we eliminate the perturbation term and let the trajectory of the diffusion model go back to that with the clear input and achieve the mimicking.

Therefore, in the short-range guidance term, we use the super-resolution operation~\cite{DPS} to map the adversarial sample to the guidance sample $y^{s} = H(x^{adv})$. The estimated image $\hat{x}_{t}$ is also performed via the super-resolution operation, which is defined as:
\begin{equation}
\begin{split}
    \nabla_{x_t} &\log p(y^{s}|x_t) = \\
    &-R_t\nabla_{x_{t}} ||H(\hat{x_{t}}) - H(x^{adv})||_1,
\end{split}
    \label{eq:Second term}
\end{equation}
where $H(\ast)$ is the super-resolution (x4) operator~\cite{DPS} and is a non-linear transform to project images of low resolution onto high resolution, which could be calculated based on the Bicubic interpolation~\cite{DPS}. The bias leads by Bicubic interpolation will increase the Manhattan distance, and thus achieve a change from the short-range distance to the long-range distance.

To apply two guidance, the guidance term in Eq.~\ref{eq:guidedmethoddefination} could be re-defined as:
\begin{equation}
\begin{split}
     \nabla_{x_t} \log p(y^{l}|x_t) + \nabla_{x_t} \log p(y^{s}|x_t),
\end{split}
    \label{eq:multiguidenceterm}
\end{equation}
where the two guidance are independent since the long-range and short-range guidance are two independent cases. Eq.~\ref{eq:multiguidenceterm} could be calculated directly based on Eq.~\ref{eq:first term} and Eq.~\ref{eq:Second term}.

Following the advice of DPS~\cite{DPS}, the guided factor is:
\begin{equation} 
    R_{t} = \frac{1}{\sigma(t)^2}.
    \label{eq:factor}
\end{equation}
Note that the proposed method tries to mimic the trajectory of the diffusion model with $x^{ori}$ input and there is no serious setting for MimicDiffusion.

\textbf{Sampling strategy}. It is difficult to let the entire reverse process be the long-range distance. Besides, calculating the gradient has a high computation cost. To alleviate these, we choose to implement the guided method in the middle phase of the reverse process of all time steps. Concretely, following the empirical finding~\cite{yu2023freedom}, the whole generation time step is in $[T::0]$. We choose the middle phase of the reverse process, i.e., $[s::e]$, to implement the guided method and vice versa not implementing the guided method, where $s = 50\%T, e=20\%T$. In this way, we avoid implementing the guided method in small time steps to avoid adding extra perturbation, reducing the time cost at the same time, and thus try to avoid implementing the guidance method on the short-range distance part.

Algorithm 1 summarizes the proposed MimicDiffusion. In our method, the hyperparameters include the $R_{t}$ and the interval $[s,e]$. For the guided factor, it could be calculated directly without additional constraints. Due to the independent property of the proposed multi-guidance, we directly use the same guided factor for both guidance terms. In the end, Yu \textit{et al.}~\cite{yu2023freedom} showed that minor numerical fluctuation under the middle phase of the reverse process for $s$ and $e$ will have little influence on the performance of the guided method. Therefore, we successfully achieve the adversarial purification without the serious setting.

\begin{algorithm}[t]
    \caption{The overall algorithm for MimicDiffusion} \label{al:stage1}
    \begin{algorithmic}[1]
     \Statex \textbf{Input:} $x^{adv}$, $T$ \Comment{Pre-trained diffusion model}
     \Statex \textbf{Output:} $x_{0}$
     
    \State $s = 50\%T, e=20\%T$ \Comment{Initialize sampling strategy}
    \State $x_{T}\sim\mathcal{N}(0,1)$
    \For{$t$ in $[T,T-1,...,1]$} 
        \State Calculate $\hat{x}_{t}$ and $x_{t-1}$ by the Reverse process
        \If{$t\in[s,e]$}
            \State Calculate $R_{t}$ by Eq.~\ref{eq:factor}
            \State $g^l\gets  \nabla_{x_t} \log p(y^l|x_t)$ by Eq.~\ref{eq:first term}
            \State $g^s\gets  \nabla_{x_t} \log p(y^s|x_t)$ by Eq.~\ref{eq:Second term}
            \State $x_{t-1} \gets x_{t-1} + g^l + g^s$ 
            \Else
                \State $x_{t-1} \gets x_{t-1}$
        \EndIf
    \EndFor \label{ref:end}
    \Statex\textbf{Return:} $x_{0}$
    \end{algorithmic}
\end{algorithm}

\section{Experiment}
\subsection{Experimental Settings}
\textbf{Datasets and network architectures}. We consider three datasets for evaluation: CIFAR-10, CIFAR-100~\cite{datasets:Cifar}, and ImageNet~\cite{ImageNet}. Meanwhile, we compare various state-of-the-art defense methods reported by the standardized benchmark: RobustBench \cite{robustbench} on CIFAR-10 and CIFAR-100 while comparing other adversarial purification methods on CIFAR-10. For classifiers, we consider two widely used backbones on RobustBench: WideResNet-28-10 and WideResNet-70-16 \cite{WideResNet}. For ImageNet, we consider the ResNet50 as the backbone. 

\textbf{Adversarial attack methods}. We evaluate our method with the common adversarial attack method: the strong adversarial attack method AutoAttack \cite{autoattack} with two settings: AutoAttack($\ell_{\infty},\epsilon=8/255$) and AutoAttack($\ell_{2},\epsilon=0.5$) respectively, projected gradient descent (PGD) attack($\ell_{\infty},\epsilon=8/255$)~\cite{PGDattack}, and C\&W attack~\cite{CWattack}. Meanwhile, to make a fair comparison with other adversarial purification methods, we evaluate our method with the adaptive attack: Backward pass differentiable approximation (BPDA+EOT)~\cite{bpda}. The experimental results for the C\&W attack and the PGD attack are reported in the supplementary material.

\textbf{Pre-trained diffusion model}. We use the unconditional CIFAR-10 checkpoint of EDM offered by NVIDIA \cite{Karras2022edm} for our method on CIFAR-10 datasets. We fine-tune the unconditional CIFAR-10 checkpoint based on CIFAR-100 for our method following the training method offered by NVIDIA \cite{Karras2022edm}. For ImageNet, we use the pre-trained diffusion model offered by Nie \textit{et al.}~\cite{adpuri:baseline3}. We evaluate our model on a single RTX4090 GPU with 24 GB memory. For CIFAR-10 and CIFAR-100, $T=100$. For ImageNet, $T=1000$.
\begin{table*}[htbp]
    \centering
    \caption{Standard accuracy and robust accuracy against AutoAttack $\ell_{\infty} (\epsilon=8/255)$ on CIFAR-10 (*methods use extra data)}
    \begin{tabular}{c c c c}
        \toprule
        \textbf{Method} & \textbf{Backbone} &\textbf{Standard Accuracy(\%)} & \textbf{Robust Accuracy(\%)} \\
        \midrule
        Zhang \textit{et al.}~\cite{baseline:zhang2020}* & WideResNet-28-10 & 89.36 & 59.96 \\
         Wu \textit{et al.}~\cite{basline:wu2020}* & WideResNet-28-10 & 88.25 & 62.11 \\
         Gowal \textit{et al.}~\cite{baseline:Gowal2020}*  & WideResNet-28-10 & 89.48 & 62.70 \\
         Wu \textit{et al.}~\cite{basline:wu2020} & WideResNet-28-10 & 85.36 & 59.18 \\
         Gowal \textit{et al.}~\cite{basline:Gowal2021} & WideResNet-28-10 & 87.33 & 61.72\\
         Rebuffi \textit{et al.}~\cite{baseline:rebuffi2021} & WideResNet-28-10 & 87.50 & 65.24 \\
        GDPM~\cite{GDPM} & WideResNet-28-10 & 84.85 & 71.18 \\
        Nie \textit{et al.}~\cite{adpuri:baseline3} & WideResNet-28-10 & 89.23 & 71.03 \\
        MimicDiffusion (Our) & WideResNet-28-10 &  \textbf{93.32} $\pm$ \textbf{2.94} & \textbf{ 92.67} $\pm$ \textbf{3.15}\\
        \midrule
        Gowal \textit{et al.}~\cite{baseline:Gowal2020}* & WideResNet-70-16  & 91.10 &  66.02 \\
        Rebuffi \textit{et al.}~\cite{baseline:rebuffi2021}* & WideResNet-70-16 & 92.23 & 68.56 \\
        Gowal \textit{et al.}~\cite{baseline:Gowal2020} & WideResNet-70-16 & 85.29 & 59.57 \\
        Rebuffi \textit{et al.}~\cite{baseline:rebuffi2021} & WideResNet-70-16 & 88.54 & 64.46 \\
        Gowal \textit{et al.}~\cite{basline:Gowal2021}& WideResNet-70-16 & 88.74 &  66.60 \\
        Nie \textit{et al.}~\cite{adpuri:baseline3} & WideResNet-70-16 & 91.04 & 71.84 \\
        MimicDiffusion (Our) & WideResNet-70-16 &\textbf{93.63} $\pm$ \textbf{2.67} &  \textbf{92.53} $\pm$ \textbf{3.06}\\
        \bottomrule
    \end{tabular}
    \label{tab:cifa10_autoattack_linf_experiment}
\end{table*}
\textbf{Evaluation metrics}. We use \textit{standard accuracy} and \textit{robust accuracy} as the evaluation metrics following the prior works \cite{adpuri:baseline3}. Meanwhile, following the experimental setting \cite{adpuri:baseline3} to reduce the computation cost of applying adaptive attacks, we evaluate the robust accuracy for all methods with BPDA+EOT attack on a fixed subset of 512 images that are randomly sampled from the test set. Meanwhile, the visualization for the purified images is reported in the supplementary material.

\subsection{Experimental Results}
We first report the results of MimicDiffusion compared with the state-of-the-art adversarial training method reported by RobustBench \cite{robustbench} against the $\ell_{\infty}$ and $\ell_{2}$ threat models, respectively.

\textbf{CIFAR-10.} Table.~\ref{tab:cifa10_autoattack_linf_experiment} shows the robustness performance against the AutoAttack with $\ell_{\infty}$ and $\epsilon=8/255$. Specifically, MimicDiffusion improves the average robust accuracy by 21.64\% on WideResNet-28-10 and by 20.69\% on WideResNet-70-16 respectively compared with the best baseline method. Meanwhile, compared with the adversarial training methods that need extra data, MimicDiffusion improves the average robust accuracy by 29.97\% on WideResNet-28-10 and by 23.97\% on WideResNet-70-16 respectively. It should be noted that our method effectively narrows the gap between standard accuracy and robust accuracy, showcasing the efficacy of mimicking the diffusion model using clean images, where the gap between standard accuracy and robust accuracy is 0.65\% and 1.1\% on WideResNet-28-10 and WideResNet-70-16 respectively.

\begin{table*}[htbp]
    \centering
    \caption{Standard accuracy and robust accuracy against AutoAttack $\ell_{2} (\epsilon=0.5)$ on CIFAR-10 (*methods use extra data)}
    \begin{tabular}{c c c c}
        \toprule
        \textbf{Method} & \textbf{Classifier} &\textbf{Standard Accuracy(\%)} & \textbf{Robust Accuracy(\%)} \\
        \midrule
        Augustin \textit{et al.}~\cite{baseline:Augustin2020}* & WideResNet-28-10 & 92.23 & 77.93 \\
        Rony \textit{et al.}~\cite{baseline:Rony2019} & WideResNet-28-10 & 89.05 & 66.41 \\
        Ding \textit{et al.}~\cite{baseline:Ding2020} & WideResNet-28-10 & 88.02 &  67.77 \\
        Wu \textit{et al.}~\cite{basline:wu2020}* & WideResNet-28-10 & 88.51 & 72.85 \\
        Sehwag \textit{et al.}~\cite{Sehwag:2021}*  & WideResNet-28-10 & 90.31 & 75.39\\
        Rebuffi \textit{et al.}~\cite{baseline:rebuffi2021} & WideResNet-28-10 & 91.79 & 78.32 \\
        GDPM & WideResNet-28-10 & 92.00 & 75.28 \\
        Nie \textit{et al.}~\cite{adpuri:baseline3} & WideResNet-28-10 & 91.38 & 78.98 \\
        MimicDiffusion (Our) & WideResNet-28-10 & \textbf{ 93.66} $\pm$ \textbf{3.22} &  \textbf{92.26} $\pm$ \textbf{3.40} \\
        \midrule
        Gowal \textit{et al.}~\cite{baseline:Gowal2020}* & WideResNet-70-16 & 94.74 &  79.88\\
        Rebuffi \textit{et al.}~\cite{baseline:rebuffi2021}* & WideResNet-70-16 & 95.74 & 81.44\\
        Gowal \textit{et al.}~\cite{baseline:Gowal2020} & WideResNet-70-16 & 90.90 & 74.03\\
        Rebuffi \textit{et al.}~\cite{baseline:rebuffi2021} & WideResNet-70-16 & 92.41 & 80.86\\
        Nie \textit{et al.}~\cite{adpuri:baseline3} & WideResNet-70-16 & \textbf{93.24} & 81.17\\
        MimicDiffusion (Our) & WideResNet-70-16 & 93.49 $\pm$ 2.77 & \textbf{92.26} $\pm$ \textbf{2.78} \\ 
        \bottomrule
    \end{tabular}
    \label{tab:cifa10_autoattack_l2_experiment}
\end{table*}

\begin{table*}[htbp]
    \centering
    \caption{Standard accuracy and robust accuracy against AutoAttack $\ell_{\infty} (\epsilon=8/255)$ on CIFAR-100 (*methods use extra data)}
    \begin{tabular}{c c c c}
        \toprule
        \textbf{Method} & \textbf{Classifier} &\textbf{Standard Accuracy(\%)} & \textbf{Robust Accuracy(\%)} \\
        \midrule
         Debenedetti \textit{et al.}~\cite{baseline:Debenedetti2022}* & XCiT-M12 & 69.21 & 34.21  \\
         Debenedetti \textit{et al.}~\cite{baseline:Debenedetti2022}* & XCiT-L12 & 70.76 & 35.08 \\
        Gowal \textit{et al.}~\cite{baseline:Gowal2020}* & WideResNet-70-16 & 69.15 & 36.88 \\
        \midrule
         Pang \textit{et al.}~\cite{baseline:pang2022} & WideResNet-28-10 & 63.66 & 31.08 \\
        Rebuffi \textit{et al.}~\cite{baseline:rebuffi2021} & WideResNet-28-10 & 62.41 & 32.06  \\
        Wang \textit{et al.}~\cite{baseline:wang2023} & WideResNet-28-10 & 72.58 & 38.83 \\
        \midrule
        Pang \textit{et al.}~\cite{baseline:pang2022} & WideResNet-70-16 & 65.56 & 33.05 \\
        Rebuffi \textit{et al.}~\cite{baseline:rebuffi2021} & WideResNet-70-16 & 63.56 & 34.64 \\
        Wang \textit{et al.}~\cite{baseline:wang2023} & WideResNet-70-16 & \textbf{75.22} & 42.67 \\
        \midrule
        MimicDiffusion (Our) & WideResNet-28-10 & 63.53 $\pm$ 6.17 & \textbf{61.35} $\pm$ \textbf{5.45} \\  
        MimicDiffusion (Our) & WideResNet-70-16 & 64.32 $\pm$ 5.77 & \textbf{62.26} $\pm$ \textbf{4.31} \\ 
        \bottomrule
    \end{tabular}
    \label{tab:cifa100_autoattack_linf_experiment}
\end{table*}

Table.~\ref{tab:cifa10_autoattack_l2_experiment} shows the robustness performance against the AutoAttack with $\ell_{2}$ and $\epsilon=0.5$. Specifically, MimicDiffusion improves the average robust accuracy by 13.28\% on WideResNet-28-10 and by 11.09\% on WideResNet-70-16 respectively compared with the best baseline method. Meanwhile, MimicDiffusion outperforms the method with the extra data. Except for the improvement in the robust accuracy, the gap between the standard accuracy and robust accuracy is still reduced. These results confirm that MimicDiffusion is effective in improving accuracy against $\ell_{2}$ threat.
\begin{table*}[htbp]
    \centering
    \caption{Standard accuracy and robust accuracy against BPDA+EOT ($\ell_{\infty},\epsilon=8/255$) on WideResNet-28-10 for CIFAR-10 }
    \begin{tabular}{c c c c}
        \toprule
        \textbf{Method}& \textbf{Purification}   &\textbf{Standard Accuracy(\%)} & \textbf{Robust Accuracy(\%)} \\
        \midrule
        Song \textit{et al.}~\cite{baseline:song2018} & Gibbs Update & 95.00 & 9.00 \\
        Yang \textit{et al.}~\cite{baseline:yang2019}* &  Mask+Recon. & 94.00 & 15.00  \\
        Hill \textit{et al.}~\cite{baseline:hill2021}*  & EBM+LD & 70.76 & 35.08 \\
        Yong \textit{et al.}~\cite{baseline:Yoon2021}*  & DSM+LD & 86.14 &  70.01 \\
        \midrule
        Nie \textit{et al.}~\cite{adpuri:baseline3}($t^{\ast} = 0.0075$) & Diffusion & \textbf{91.38} & 77.62 \\
        NIe \textit{et al.}~\cite{adpuri:baseline3}($t^{\ast} = 0.1$) & Diffusion & 89.23 & 81.56 \\
        GDPM~\cite{GDPM} & Diffusion & 90.36 & 77.31\\
        MimicDiffusion (Our) & Diffusion & 92.5 $\pm$ 5.12 & \textbf{92.00} $\pm$ \textbf{6.1} \\
        \bottomrule
    \end{tabular}
    \label{tab:cifa10_bpda_experiment}
\end{table*}
To sum up, the experimental results on CIFAR-10 show the effectiveness of MimicDiffusion in defending against $\ell_{\infty}$ and $\ell_{2}$ threat models on CIFAR-10. Meanwhile, MimicDiffusion keeps a high level of standard accuracy, which demonstrates the validity of MimicDiffusion.

\textbf{CIFAR-100}. Table. \ref{tab:cifa100_autoattack_linf_experiment} shows the robustness performance against the AutoAttack with $\ell_{\infty}$ and $\epsilon=8/255$. It can be found that MimicDiffusion improves the average robust accuracy by 18.68\% on WideResNet-28-10 and by 19.59\% on WideResNet-70-16 respectively. Even with the minimum level of MimicDiffusion, we still improve the average robust accuracy by 13.23\% on WideResNet-28-10 and by 15.28\% on WideResNet-70-16. Meanwhile, based on the different datasets, our method tends to reduce the gap between standard accuracy and robust accuracy. These results prove that our model achieves the mimicking and significantly outperforms other baseline models on the CIFAR-100 dataset.

\textbf{ImageNet}. We report the extra experimental results on ImageNet~\cite{ImageNet} shown in the Supplementary. These results follow the experimental setting in Nie \textit{et al.} \cite{adpuri:baseline3}. According to the results, we still see a significant improvement, with the average robust accuracy increasing by 17.64\%. Meanwhile, MimicDiffusion successfully reduces the gap between standard accuracy and robust accuracy. 

Overall, experimental results demonstrate that MimicDiffusion achieves a large improvement in defending the $\ell_{\infty}$ and $\ell_{2}$ threat model on CIFAR-10 and defending the $\ell_{\infty}$ threat model on CIFAR-100 and ImageNet. This also demonstrates the effectiveness of the proposed method for adversarial purification.

\textbf{Defense against unseen threats}. To demonstrate the effectiveness of MimicDiffusion, we compare it with other adversarial purification methods using BPDA+EOT ($\ell_{\infty},\epsilon=8/255$)~\cite{bpda}. This attack, which is adaptive and stochastic, is specifically designed for purification methods, as some adversarial purification methods are not compatible with AutoAttack. The experimental results are shown in Table. \ref{tab:cifa10_bpda_experiment}. It can be found that MimicDiffusion also gets the best performance in terms of robust accuracy, which improves the average robust accuracy by almost 4.35\% in the worst performance. Meanwhile, we find that there is less gap between standard accuracy and robust accuracy for MimicDiffusion, which demonstrates the success of MimicDiffusion.
\subsection{Ablation Study}

As shown in Table.~\ref{tab:alabation study}, we list the different combinations among MimicDiffusion. To prove the necessity for two guidance, it can be found that using each of the guidance individually results in a significant decrease in robust accuracy. Then, Compared with the $\ell_{2}$ norm, widely used in previous methods, $\ell_{1}$ norm could increase $31.93\%$ robust accuracy and achieve a large improvement, which proves the effectiveness of Lemma~\ref{lemma1}. Meanwhile, the super-resolution guidance should change the short-range to the long-range. This means that using the $g^s$ in addition to $g^l$ significantly improves the robust accuracy by almost 66.32\%. The sampling strategy improves robust accuracy by 3.5\% without impacting standard accuracy, and it performs well. We also reported the time cost of using the sampling strategy in the supplementary material.
\begin{table}[htbp]
    \centering
    \begin{tabular}{c c c c | c c}
    \toprule
    \textbf{$g^l$} & \textbf{$g^s$} & \textbf{$d$} & \textbf{Sampling} &\textbf{Standard(\%)} & \textbf{Robust(\%)} \\
    \midrule
     $\surd$ & & $\ell_{2}$ &  & 86.42 & 16.7  \\
     & $\surd$ & $\ell_{2}$&  & 10.89 & 10.00 \\
     $\surd$ & & $\ell_{1}$&  & 88.39 & 22.78  \\
     & $\surd$ & $\ell_{1}$&  & 10.80 & 7.10  \\
    $\surd$ & $\surd$ & $\ell_{2}$&  & 91.03 & 57.07  \\
    $\surd$ & $\surd$ & $\ell_{1}$&  & 91.30 & 89.10  \\
    $\surd$ & $\surd$ & $\ell_{1}$& $\surd$ & 93.30 & 92.60  \\
    \bottomrule
    \end{tabular}
    \caption{Ablation study for MimicDiffusion based on CIFAR-10 against AutoAttack($\ell_{\infty},\epsilon=8/255$) with WideResNet-28-10, where Sampling is sampling strategy, standard is standard accuracy, and robust is robust accuracy.}
    \label{tab:alabation study}
\end{table}
\section{Conclusion}
We proposed a new defense method called MimicDiffusion to achieve adversarial purification by mimicking the trajectory of the diffusion model using clean images as the input without serious settings. Specifically, using the two proposed guidance methods with the Manhattan distance can mitigate the negative impact caused by adversarial perturbation, as mentioned in Lemma~\ref{lemma1}. Meanwhile, to avoid introducing extra perturbation, we proposed a novel sampling strategy based on the property of the guided method, which can also reduce the time cost. To show the robust performance of our method, we conducted thorough experiments on CIFAR-10, CIFAR-100, and ImageNet. These experiments involved using various classifier backbones, such as WideResNet-70-16, WideResNet-28-10, and Resnet50, to compare our method with state-of-the-art adversarial training and adversarial purification methods. The experimental results showed that our method got the best performance in defense of various strong adaptive attacks such as AutoAttack, PGD attack, C\&W attack, and BPDA+EOT. These results show that MimicDiffusion could mimic the trajectory of the diffusion model using the clean image as the input.

Despite the large improvement, the proposed two guidance require calculating the gradients and will increase the computation cost. We will explore finding a gradient-free guided method in further work. Furthermore, in addition to focusing on small steps, we will also explore alternative distance metrics to mitigate the impact of adversarial perturbations.

\bibliographystyle{unsrt}  
\bibliography{main}  
\clearpage
\setcounter{page}{1}
\section{Appendix}

\begin{table}[htbp]
    \centering
    \begin{tabular}{c|c}
    \toprule
    \textbf{Sampling} & \textbf{Time(s)}\\
    \midrule
          & 268 \\
      $\surd$ & 156 \\
    \bottomrule
    \end{tabular}
    \caption{Alabation studying of time cost for purifying one image in ImageNet, where sampling is the sampling strategy.}
    \label{tab:time cost}
\end{table}
\begin{table*}[htbp]
    \centering
    \caption{Standard accuracy and robust accuracy against AutoAttack $\ell_{\infty} (\epsilon=8/255)$ on ImageNet}
    \begin{tabular}{c c c c}
        \toprule
        \textbf{Method} & \textbf{Classifier} &\textbf{Standard Accuracy(\%)} & \textbf{Robust Accuracy(\%)} \\
        \midrule
        Wang \textit{et al.}~\cite{robustbench} & ResNet50 & 62.56 & 31.06  \\
        Wong \textit{et al.}~\cite{baseline:Wong2020} & ResNet50 & 55.62 & 26.95 \\
        Salman \textit{et al.}~\cite{baseline:Salman2020} & ResNet50 & 64.02 &  37.89 \\
        Bai \textit{et al.}~\cite{baseline:bai2021} & ResNet50 & 67.38 & 35.51 \\
        Nie \textit{et al.} & ResNet50 & \textbf{68.22} & 43.89  \\
        MimicDiffusion (Our) & ResNet50 & 66.92 $\pm$ 10.44 &  \textbf{61.53} $\pm$ \textbf{9.7}\\
        \bottomrule
    \end{tabular}
    \label{tab:imagenet_autoattack_linf_experiment}
\end{table*}
\begin{table*}[htbp]
    \centering
    \caption{Standard accuracy and robust accuracy against PGD $\ell_{\infty} (\epsilon=8/255)$ on CIFAR-10}
    \begin{tabular}{c c c c}
        \toprule
        \textbf{Method} & \textbf{Classifier} &\textbf{Standard Accuracy(\%)} & \textbf{Robust Accuracy(\%)} \\
        \midrule
        Nie \textit{et al.} & WideResNet-70-16 &  91.03 & 57.69  \\
        Wang \textit{et al.}~\cite{RobustEvaluation}& WideResNet-70-16 & 90.67 & 63.52 \\
        MimicDiffusion (Our) & WideResNet-70-16 & \textbf{92.05} $\pm$ 6.02 &  \textbf{91.55} $\pm$ 6.84\\
        \midrule
        GDPM & WideResNet-28-10& 93.50 & 90.10 \\
        Nie \textit{et al.} & WideResNet-28-10 & 91.00 & 54.92  \\
        Wang \textit{et al.}~\cite{RobustEvaluation}& WideResNet-28-10 & 90.70 &62.15 \\
        MimicDiffusion (Our) & WideResNet-28-10 & \textbf{91.93} $\pm$ 6.00 &  \textbf{91.88} $\pm$ 6.01 \\
        \bottomrule
    \end{tabular}
    \label{tab:pgd}
\end{table*}
\begin{table*}[htbp]
    \centering
    \caption{Standard accuracy and robust accuracy against C\&W Attack  $\ell_{2} (\epsilon=8/255), \text{EOT}=50$ on CIFAR-10}
    \begin{tabular}{c c c c}
        \toprule
        \textbf{Method} & \textbf{Classifier} &\textbf{Standard Accuracy(\%)} & \textbf{Robust Accuracy(\%)} \\
        \midrule
        Nie \textit{et al.} & WideResNet-70-16 & \textbf{92.35} & 47.00  \\
        MimicDiffusion (Our) & WideResNet-70-16 & 91.85 $\pm$ 6.46 &  \textbf{91.67} $\pm$ \textbf{6.49}\\
        \midrule
        GDPM & WideResNet-28-10&  21.30 & 21.71\\
        Nie \textit{et al.} & WideResNet-28-10 & \textbf{93.53} & 47.65  \\
        MimicDiffusion (Our) & WideResNet-28-10 & 90.34 $\pm$ 6.2 &  \textbf{89.91} $\pm$ \textbf{6.5}\\
        \bottomrule
    \end{tabular}
    \label{tab:cw1}
\end{table*}
\section{Extra Experiment}

\subsection{ImageNet}
To further prove the validity of MimicDiffusion, we report the extra experimental results on ImageNet~\cite{ImageNet} shown in Table. \ref{tab:imagenet_autoattack_linf_experiment} following the experimental setting in Nie \textit{et al.} \cite{adpuri:baseline3}. According to the results, we still get a large improvement, which improves the average robust accuracy by 17.64\%. Meanwhile, MimicDiffusion successfully reduces the gap between standard accuracy and robust accuracy. 
\subsection{Additional Attack Method}
We report the additional experimental results on PGD and C\&W attack based on CIFAR-10 compared with the latest works for adversarial purification including GDPM, \cite{adpuri:baseline3}, and \cite{RobustEvaluation}. To make sure of a fair comparison, we test our method following the advice of~\cite{RobustEvaluation}, and the results are shown in Table.~\ref{tab:pgd} and Table.~\ref{tab:cw1}. It can be noticed that our method achieves the best performance under average robust accuracy against PGD and C\&W attacks. Concretely, when against the C\&W attack, we improved almost 35.76\% average robust accuracy in the worst condition. When against the PGD attack, we improved almost 1\% average robust accuracy. Meanwhile, MimicDiffusion significantly reduces the gap between standard accuracy and robust accuracy. In this way, we prove that the proposed MimicDiffusion could reduce the negative influence of adversarial perturbation and avoid adding too much extra noise. Meanwhile, the performance of MimicDiffusion is stable against different attack methods based on the same setting.

\subsection{Visualization}
\begin{figure*}
    \centering
    \includegraphics{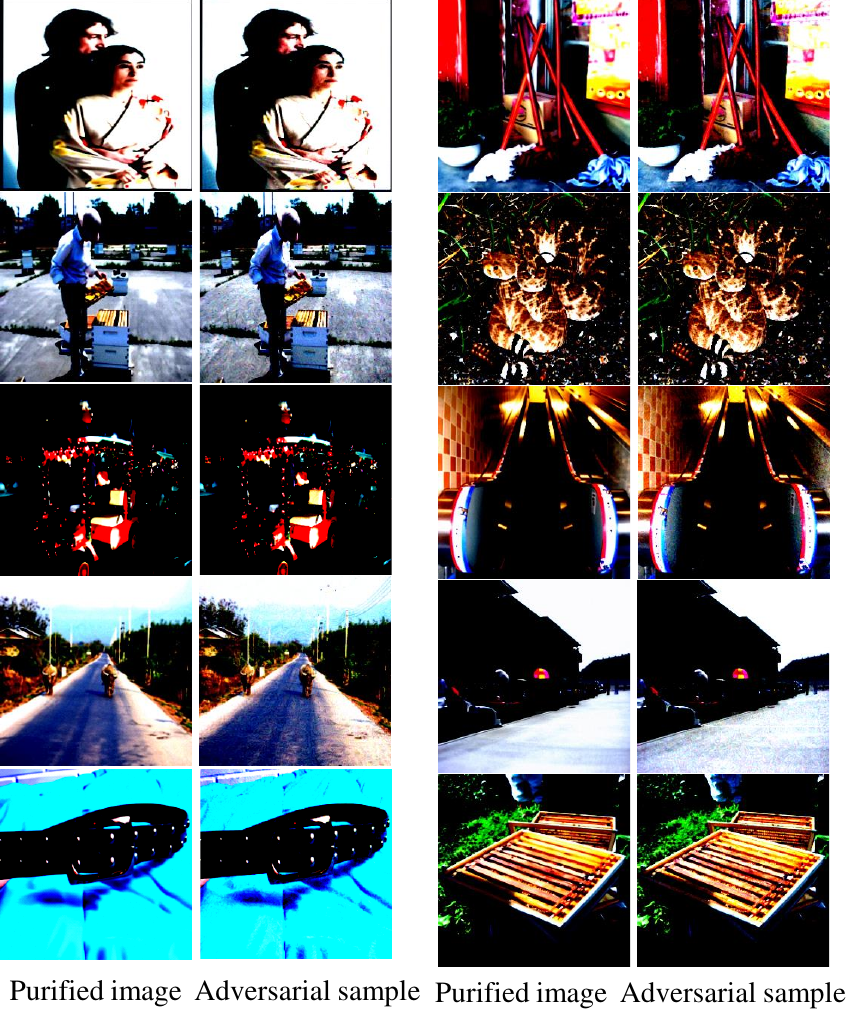}
    \caption{Visualization of MimicDiffusion against AutoAttack $\ell_{\infty} (\epsilon=8/255)$}
    \label{fig:purifed}
\end{figure*}
To show the ability of purification, we report the purified images shown in Fig.~\ref{fig:purifed}. It can be noticed that MimicDiffusion could successfully purify the adversarial perturbation and keep the label semantic as much as possible. 
\subsection{Time Cost}
To further prove the necessity of the sampling strategy, we make an ablation study for the time cost of using the sampling strategy shown in Table.~\ref{tab:time cost}. It can be noticed that the sampling strategy reduces half of the time cost compared with implementing the guided method in the entire reverse process.

\end{document}